\documentclass{article}

\usepackage{microtype}
\usepackage{graphicx}
\usepackage{subcaption}
\usepackage{booktabs} 

\usepackage{hyperref}

\usepackage{icml2026}

\usepackage{amsmath}
\usepackage{amssymb}
\usepackage{mathtools}
\usepackage{amsthm}
\usepackage{xcolor}
\definecolor{darkgreen}{RGB}{0,100,0}

\usepackage[capitalize,noabbrev]{cleveref}

\theoremstyle{plain}
\newtheorem{theorem}{Theorem}[section]
\newtheorem{proposition}[theorem]{Proposition}
\newtheorem{lemma}[theorem]{Lemma}

\theoremstyle{definition}

\theoremstyle{remark}

\icmltitlerunning{Scaling Tasks, Not Samples: Mastering Humanoid Control through Multi-Task Model-Based Reinforcement Learning}

\begin{document}

\twocolumn[
  \icmltitle{Scaling Tasks, Not Samples: Mastering Humanoid Control through Multi-Task Model-Based Reinforcement Learning}

  \begin{center}
  \begin{tabular}{@{}l@{\hspace{1.2cm}}l@{\hspace{1.2cm}}l@{\hspace{1.2cm}}l@{}}
  \textbf{Shaohuai Liu}\textsuperscript{1*} &
  \textbf{Weirui Ye}\textsuperscript{2*} &
  \textbf{Yilun Du}\textsuperscript{3$\dagger$} &
  \textbf{Le Xie}\textsuperscript{3$\dagger$} \\
  \small Texas A\&M University &
  \small MIT &
  \small Harvard University &
  \small Harvard University \\
  \small \texttt{liushaohuai5@tamu.edu} &
  \small \texttt{ywr@csail.mit.edu} &
  \small \texttt{ydu@seas.harvard.edu} &
  \small \texttt{xie@seas.harvard.edu} \\
  \end{tabular}
  \end{center}

  \begin{center}
  \textsuperscript{*}Equal contribution \quad
  \textsuperscript{$\dagger$}Equal advising
  \end{center}

  \vskip 0.3in
]



\printAffiliationsAndNotice{}  

\begin{abstract}
Developing generalist robots capable of mastering diverse skills remains a central challenge in embodied AI. While recent progress emphasizes scaling model parameters and offline datasets, such approaches are limited in robotics, where learning requires active interaction. We argue that effective online learning should scale the \emph{number of tasks}, rather than the number of samples per task.
This regime reveals a structural advantage of model-based reinforcement learning (MBRL). Because physical dynamics are invariant across tasks, a shared world model can aggregate multi-task experience to learn robust, task-agnostic representations. In contrast, model-free methods suffer from gradient interference when tasks demand conflicting actions in similar states. Task diversity therefore acts as a regularizer for MBRL, improving dynamics learning and sample efficiency.
We instantiate this idea with \textbf{EfficientZero-Multitask (EZ-M)}, a sample-efficient multi-task MBRL algorithm for online learning. Evaluated on \textbf{HumanoidBench}, a challenging whole-body control benchmark, EZ-M achieves state-of-the-art performance with significantly higher sample efficiency than strong baselines, without extreme parameter scaling. These results establish task scaling as a critical axis for scalable robotic learning. The project website is available \href{https://yewr.github.io/ez_m/}{here}.
\end{abstract}

\section{Introduction}

Generalist agents capable of robust generalization and adaptation remain a central objective in embodied AI research. Inspired by the success of Large Language Models (LLMs), recent trends in robotics have largely prioritized the Foundation Model paradigm—scaling model parameters and aggregating massive offline datasets~\citep{reed2022generalist,brohan2022rt,kim2024openvla,black2410pi0,o2024open}. While effective for broad task coverage, this approach faces a fundamental limitation unique to robotics: unlike passive text generation, embodied intelligence requires active interaction with physical environments. Relying solely on offline data limits an agent’s ability to correct errors, adapt to dynamic changes, and refine behaviors through feedback. Consequently, scaling online learning—improving the efficiency and efficacy of active interaction—is as critical as scaling offline data.

\begin{figure}[t]
    \vskip -0.2cm
    \centering
    \includegraphics[width=0.85\linewidth]{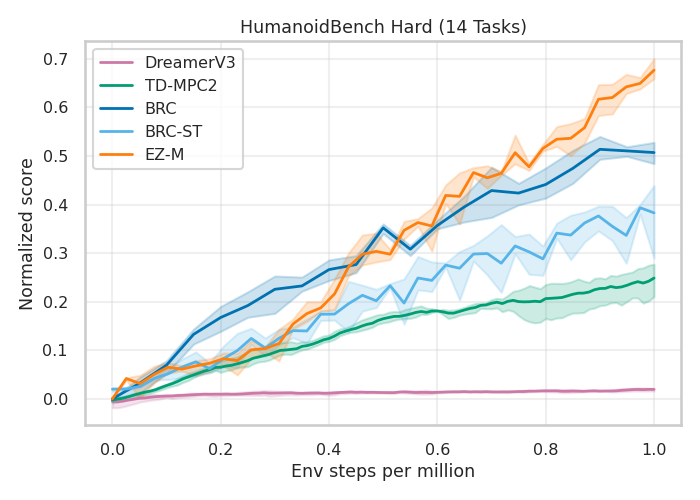}
    \caption{\textbf{Normalized task-average scores on HumanoidBench-Hard}. EZ-M matches and surpasses the strong baselines with environment interactions limited to 1 million. All runs are with 3 random seeds. EZ-M significantly outperforms all baselines.}
    \label{fig:norm_scores}
    \vskip -0.2cm
\end{figure}

We propose that effective online learning requires scaling the breadth of tasks, a paradigm we term online task scaling, rather than merely increasing sample counts per task. This shift exposes a critical structural advantage for Model-Based Reinforcement Learning (MBRL). While Model-Free (MF) policies often suffer from gradient interference when distinct tasks demand conflicting actions in similar states, MBRL focuses on learning the underlying transition dynamics. Since physical laws are task-invariant, a shared world model aggregates diverse multi-task data to refine its understanding of the environment. In this context, task diversity serves as a dynamics regularizer: it prevents the model from overfitting to the narrow state distributions of single tasks, ensuring the learned physics generalize globally. Consequently, MBRL effectively converts multi-task scaling into robust representation learning.

In this work, we introduce EfficientZero-Multitask (EZ-M), a sample-efficient multi-task RL algorithm building upon the EfficientZero architecture~\citep{ye2021mastering}. EZ-M incorporates targeted architectural modifications and consistency regularization to effectively exploit dynamics invariance in online settings. We evaluate our approach on HumanoidBench~\citep{sferrazza2024humanoidbench}, a high-dimensional domain where modeling complex contact-rich dynamics constitutes the primary bottleneck.

Our empirical results demonstrate that EZ-M not only achieves State-of-the-Art (SoTA) performance but also exhibits distinct positive transfer properties: whereas Model-Free (MF) baselines degrade or plateau as task cardinality increases, EZ-M leverages task diversity to improve sample efficiency. This suggests a scalable trajectory for robotic learning: apart from relying exclusively on large-scale offline datasets, we envision a regime of massive parallel online learning, where diverse tasks synergize to construct a unified, generalized world model. And our primary contributions are as follows: 

\begin{itemize} 
    \item We propose EfficientZero-Multitask (EZ-M), extending the sample efficiency of MBRL to the multi-task online setting via novel consistency mechanisms that mitigate negative interference between tasks. 
    \item We provide a theoretical analysis formalizing the structural advantage of MBRL. We show that while MF objectives suffer from gradient conflict in multi-task settings, MBRL benefits from dynamics invariance, yielding asymptotically superior sample efficiency as the number of tasks increases. 
    \item We achieve State-of-the-Art performance on the comprehensive HumanoidBench, demonstrating that online task scaling in MBRL offers a viable pathway for high-dimensional robotic control. \end{itemize}
    
\section{Related Work}

\textbf{Humanoid Control}. Humanoid control remains a rigorous benchmark due to under-actuation and contact complexity. While prior work spans imitation-based tracking~\citep{peng2018deepmimic} and adversarial RL~\citep{peng2021amp}, often leveraging GPU-accelerated simulation~\citep{makoviychuk2021isaac}, comprehensive benchmarks like HumanoidBench~\citep{sferrazza2024humanoidbench} highlight persistent challenges in stability and sample efficiency for online learning.

\textbf{Multi-Task Learning.} To mitigate challenges like gradient interference~\citep{yu2020gradient} and heterogeneous reward scales~\citep{hessel2019multi}, prior work employs techniques ranging from gradient projection~\citep{yu2020gradient,chen2018gradnorm} to distillation~\citep{teh2017distral}. Parallel efforts leverage high-capacity Vision-Language-Action (VLA) models~\citep{kim2024openvla}, applying supervised learning or offline RL to massive teleoperation datasets~\citep{o2024open} to achieve generalization. Diverging from these resource-intensive paradigms, we prioritize \textbf{sample-efficient online learning}. Rather than relying on massive model scaling or pre-collected datasets, we utilize a compact shared dynamics model for MBRL. We demonstrate that structural regularization in dynamics and value learning can effectively substitute for sheer scale, enabling stable, task-scalable online RL within a limited parameter budget.

\textbf{Model-Based RL.} Prior work establishes the efficacy of planning via learned dynamics, ranging from experience augmentation~\citep{sutton1991dyna} and latent rollouts~\citep{hafner2019dream,hafner2023mastering} to search-based trajectory selection~\citep{hansen2022temporal,ye2021mastering} and value re-estimation~\citep{schrittwieser2020mastering,wang2025bootstrapped}. Recent advances further integrate diffusion models~\citep{hafner2025training}. However, these innovations primarily target single-task performance. The limited multi-task MBRL literature restricts itself to MPC approaches~\citep{hansen2023td,hansen2025learning}, leaving the potential of MCTS for online multi-task learning unexplored.
\section{Preliminaries}
\textbf{Multi-Task Reinforcement Learning} (MTRL) aims to train a single agent over a set of tasks $\mathcal{T}$. Each task $\tau\in\mathcal{T}$ introduces an MDP $\mathcal{M}_\tau=(\mathcal{S}_\tau,\mathcal{A}_\tau,P_\tau,r_\tau,\gamma)$ with shared discount factor $\gamma$ but task-dependent state space $\mathcal{S}_\tau$, action space $\mathcal{A}_\tau$, state transition $s_{\tau,t+1}=P_\tau(s_{\tau,t},a_{\tau,t})$, and reward function $r_\tau$. We learn a task-conditioned policy $\pi_\theta(a|s,\tau)$ maximizing the expected return averaged over $\mathcal{T}$, mathematically defined as
$
J(\theta)=\mathbb{E}_{\tau\in\mathcal{T}}\left[\mathbb{E}_{\pi_\theta(\cdot|s,\tau)}\sum_{t\ge0}\gamma^t r_\tau(s_{\tau,t},a_{\tau,t})\right]
$
Compared to single-task RL, MTRL must fit heterogeneous learning signals with a shared representation, which can cause gradient conflicts, negative transfer, and instability, especially for shared critics under different reward/value scales. Practical training typically mixes experience from multiple tasks via task-balanced sampling and task conditioning, while controlling cross-task interference through normalization or other stabilization techniques.

\textbf{Monte Carlo Tree Search and Gumbel Search} are two tree search technologies used in MuZero series and EfficientZero series for producing improved policy and filtered actions. MCTS~\citep{coulom2006efficient,kocsis2006bandit} and grows a search tree by iteratively expanding new nodes on leaves and recursively updating the values of nodes on the dive path. It dives along the path with node descents that have the highest upper confidence bound scores~\citep{auer2002finite}. MCTS produces an improved policy using the visit count distribution of the root descents and a filtered action that is most visited. Therefore, MCTS typically requires considerable node expansions to generate stable, improved policies, and the inability to parallelize also leads to computational efficiency issues~\citep{ye2022spending}. In contrast, Gumbel search~\citep{danihelka2022policy} utilizes Gumbel sample-without-replacement~\citep{kool2019stochastic} and sequential halving~\citep{karnin2013almost} to iteratively dichotomize action spaces. 
It replaces count-based UCB for node selection with Gumbel scores, which are more reliant on value estimations. The benefits of these modifications are not only reduced node expansion, but more importantly, guarantee policy improvements.

\section{Methods}
To train an RL agent across multiple tasks, the algorithm is supposed to satisfy: (1) scale with the number of tasks increases; (2) be adaptive to task-varying observation and action spaces; (3) be robust to task-varying reward and value scales; (4) be computationally efficient in terms of both training and inference. EZ-M is built upon the success of EfficientZero-v2, which proved to be sample and computation-efficient across tasks with visual/proprio inputs and continuous/discrete actions. We extend EfficientZero-v2 to the online, multi-task RL, with improved designs that stabilize multi-task learning as listed below. Our methods could be summarized in Fig.\ref{fig:overview}.

\begin{figure*}[t]
    \centering
    \vskip -.2cm
    \includegraphics[width=0.75\linewidth]{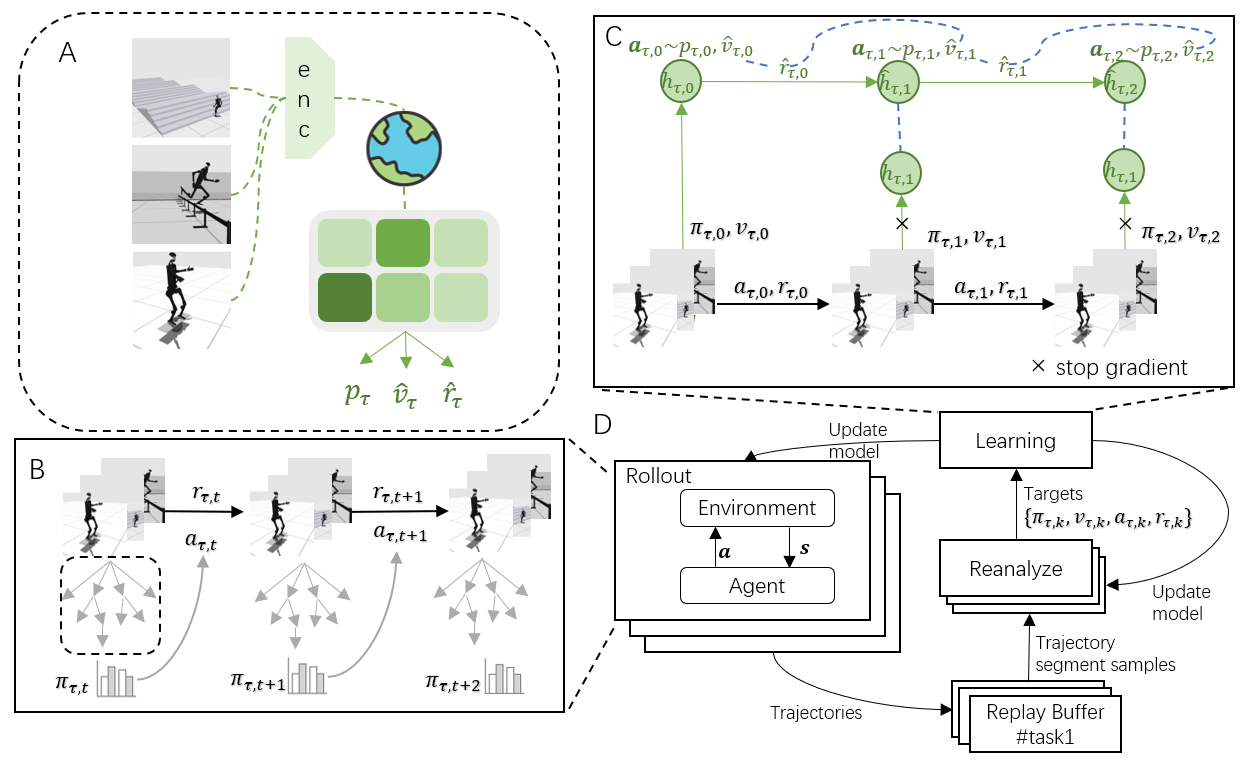}
    \caption{\textbf{Overview of EfficientZero-Multitask (EZ-M)}. \textcolor{darkgreen}{\textbf{Green}} arrows and notations represent model inference and predictions. \textbf{Black} notations and bold arrows indicate target/grounded values and transitions. 
    (A) \textbf{Task-sharing model architecture}. EZ-M uses a shared model to predict policy, value, and reward while training multiple tasks, with each component conditioned by task embeddings.
    (B) \textbf{Balanced data collection}. Rollout workers sample tasks with a balanced schedule and execute the corresponding task policy to interact with the environment. This prevents learning collapse from data imbalance and maintains diverse task coverage over time. 
    (C) \textbf{Multi-task rollout learning}. For each sampled trajectory segment, the model is unrolled in latent space conditioned on the task index to produce \(\hat{r}_{\tau,k}\), \(\hat{v}_{\tau,k}\), and \(\hat{p}_{\tau,k}\) along the imagined rollout. Supervision comes from grounded transitions \((a_{\tau,k}, r_{\tau,k})\) and reanalyzed targets \((\pi_{\tau,k}, v_{\tau,k})\), where cross-entropy loss enables a single network to learn consistent predictions across tasks. \textcolor{blue}{\textbf{Blue}} dashed lines represents temporal consistency and path consistency.
    (D) \textbf{Distributed implementation}. Data collection (Rollout) feeds trajectories into task-independent replay buffers, while Reanalyze periodically recomputes search-based targets with the latest model. The learner consumes reanalyzed batches to update parameters and broadcasts updated weights back to rollout/reanalyze workers for scalable asynchronous training.
    }
    \vskip -.2cm
    \label{fig:overview}
\end{figure*}
 
\subsection{Model, Search and Learning}
\textbf{Model components}. EZ-M learns to plan over action candidates in a hidden state space by jointly conducting temporal consistency modeling, categorical reward/value prediction, and real/imagined TD-learning. In contrast to the methods like Dreamer series~\citep{hafner2019dream,hafner2020mastering,hafner2023mastering} trained to reconstruct the original input, EZ-M didn't employ the reconstruction loss and kept the model learning to be control objective concentrated. Different from EfficientZero-v2, EZ-M uses a separate reward prediction path identical to state transition prediction rather than connecting a reward prediction head after the next state prediction $r_{t}=\mathcal{R}(h_{t+1})$. To avoid the interference of task-varying action dimensions and scales, an action encoder projects actions from different tasks into the same action embedding space to stabilize task-sharing model training. 
Specifically, the EZ-M architecture consists of six learned sub-networks:
\begin{equation}
\label{eq:model_components}
\begin{aligned}
    h_{\tau,t} &= \mathcal{H}(s_{\tau,t}), & u_{\tau,t} &= \mathcal{U}(a_{\tau,t}), \\
    \hat{h}_{\tau,t+1} &= \mathcal{G}(h_{\tau,t}, u_{\tau,t}), & \hat{r}_{\tau,t} &= \mathcal{R}(h_{\tau,t}, u_{\tau,t}), \\
    \hat{v}_{\tau,t} &= \mathcal{V}(h_{\tau,t}), & p_{\tau,t} &= \mathcal{P}(h_{\tau,t}),
\end{aligned}
\end{equation}
where $\tau$ denotes the task identifier and $t$ represents the time step. The variables $h_{\tau,t}$ and $u_{\tau,t}$ correspond to the latent state representation and action embedding derived from the raw state $s_{\tau,t}$ and action $a_{\tau,t}$, respectively. The dynamics network $\mathcal{G}$ autoregressively predicts the subsequent latent state $\hat{h}_{\tau,t+1}$, while $\hat{r}_{\tau,t}$ and $\hat{v}_{\tau,t}$ estimate the immediate reward and value function.

\textbf{Search with model priors}. We use the above model outputs as prior information to guide the tree search, and apply Gumbel  Trick~\citep{danihelka2022policy,wang2024efficientzero}.

At the root node, we calculate the Gumbel score $g(a)$ of each root descent with action $a$ using model outputs, where
\begin{equation}
    \begin{aligned}
        g(a)&=p(a)+\sigma(\hat{q}(a))\\
    \hat{q}(a)&=\mathbb{E}_{\{s,a\}\sim\mathcal{G},\mathcal{P}}\left[\sum_{t\ge0}^H\gamma^t\hat{r}+\gamma^H \hat{v}\right]
    \end{aligned}
\end{equation}
where $\sigma$ is a monotonous function following~\cite{danihelka2022policy}.  $\{s,a\}$ represents all state-action sets of the search tree with the limited node expansion budget, following the predictions of the dynamic network and policy network. $H$ represents the length of each dive trajectory during node expansion. $\hat{q}(a)$ is the visit count-average of bootstrapped value estimation along the dive paths, normalized to [0,1] using min-max. Sequential Halving~\citep{karnin2013almost} is applied to iteratively dichotomize the root action candidate set after equally expanding the remaining descents, defined as
\begin{equation}
    \mathcal{A}_n=\mathop{\text{argtop}}_{|\mathcal{A}_{n-1}|/2}\ g(a)
\end{equation}
Then it will finally produce a filtered action $A$ and an improved policy $\pi$ via
\begin{gather}
    A=\mathop{\arg\max}_{a\in\mathcal{A}_\text{remain}}\ g(a)\\
    \pi:=(\mathcal{A},\text{softmax}(g),A)
\end{gather}

At non-root nodes, the action selection follows the proposed policy $p$ to avoid unexpected exploration noises finally introduced to root value estimations. This result in a principle that visit count matches policy prior.

\textbf{Policy learning}. In contrast to max-Q policy gradient applied in actor-critic model-based RL algorithms~\citep{hansen2022temporal,hansen2023td}, we use supervised learning that let the policy head output $p$ to mimic the improved policy $\pi$ generated from the tree search. A key concern here is that EZ series are state-value based, thus it is not natural to apply max-Q learning. 
A related work demonstrates such distillation of improved policy can lead to improved performance in TD-MPC framework~\cite{wang2025bootstrapped}. 

We uses full reanalyze, applying tree search on each sampled states from the experience replay, to generate the target policy since EZ-M follows the off-policy RL setting. To be specific, the improved policy $\pi$ consists of sampled action candidate set $\mathcal{A}$, their corresponding Gumbel scores $g$, and a filtered action $A$. We found that a simple Maximum Likelihood Estimation (MLE) loss, defined as
\begin{equation}
    \mathcal{L}_p=-\log p(A)
\end{equation}
produces more stable performance in practice, especially in tasks with high dimensional action spaces.
 
\textbf{Reward prediction}. Accurately predicting environmental rewards under the situation that reward scales vary significantly with tasks is a critical challenge in multi-task model-based RL. Similar problems of task-varying Q-value scales induced by heterogeneous reward scales have placed obstacles in multi-task model free RL~\citep{hessel2019multi,nauman2025bigger}. In contrast to applying complex reward normalization methods like those work, we rather solely uses categorical modeling in reward prediction~\citep{schrittwieser2020mastering} without normalization or parameter scaling. ~\citep{bellemare2017distributional} also reports cross entropy loss over categorical modeling produces more balanced gradients and loss scales since the classification loss is not sensitive to scale changes.

\textbf{Value learning}. Learning to predict value precisely is critical to our off-policy, value-based RL algorithm. Therefore, we applied a bunch of improvements on value prediction. We keep to use categorical modeling in value estimation since we didn't use reward normalization, and it's proved to be effective in predicting task-varying values in multi-task actor-critic RL\citep{nauman2025bigger}. We also use value ensemble proposed in \citep{hansen2023td} to stablize value prediction and alleviate overestimation. Multiple value heads also bring a improvement in mixed value targets proposed in \citep{wang2024efficientzero}, which uses model-based value estimation, the root bootstrapped values of reanalyze tree search, on stale transitions sampled from the replay buffer. \citep{wang2024efficientzero} uses a rule-based method on determining when to use model-based targets and what is stale transition. We can derive the value predicting variance with the multiple value estimations, previously used for designing intrinsic reward and increasing exploration~\citep{osband2016deep,pathak2017curiosity}, and now we use the value variance as the weight balance between TD targets and model-based targets, defined as $v_\text{mix}=\alpha v_\text{td}+(1-\alpha)v_\text{model}$, where the weight coefficient $\alpha=(\text{Var}(\hat{v})-\underline{\text{Var}(\hat{v})})/(\overline{\text{Var}(\hat{v})}-\underline{\text{Var}(\hat{v})})$. $\overline{\text{Var}(\hat{v})}$ and $\underline{\text{Var}(\hat{v})}$ are the empirical maximum and minimum value variance. As the value variance goes higher, particularly in those fresh transitions, it is supposed to rely more on the TD-targets. The variance on those stale transitions is comparatively low, so we can use more model-based targets.

\textbf{Temporal Consistency on Latent States}. To stabilize multi-step rollouts and mitigate representation drift during imagination, we impose a temporal consistency regularization on the latent space. Intuitively, this constraint enforces alignment between the predicted latent state (generated by the dynamics) and the encoded latent state (derived directly from the corresponding future observation). By encouraging these two pathways to converge, we ensure that imagined trajectories remain anchored to the representation induced by real data. Concretely, we minimize the projection distance (e.g., negative cosine similarity~\citep{chen2021exploring}) between the predicted transition and the target representation, utilizing a stop-gradient on the target to prevent collapse. This regularization is critical for multi-task learning, as it prevents the dynamics model from diverging under the distribution shifts introduced by diverse task mixing.

\textbf{Base training objective}. In summary, we formulate the overall training objective as the sum of the reward prediction loss, the policy (or search-improved policy) loss, the value prediction loss, and the latent consistency regularization over $K$ unrolled steps:
\begin{equation}
\begin{aligned}
\mathcal{L}^\text{base}=\sum_{k=0}^K \Big( &\mathcal{L}_\text{CE}(r_k,\hat{r}_k)+\mathcal{L}_p(\pi_k,p_k)\\
+&\mathcal{L}_\text{CE}(v_k,\hat{v}_k)+\mathcal{L}_\text{SimSiam}(h_k,\hat{h}_k)\Big)
\end{aligned}
\end{equation}
where $r_k,\pi_k,v_k$ denote the training targets at unroll step $k$ (e.g., bootstrapped $n$-step reward/value targets and the search policy target), and $\hat{r}_k,p_k,\hat{v}_k$ are the corresponding network predictions. The term $\mathcal{L}_\text{SimSiam}(\cdot)$ collects temporal consistency regularization on latent dynamics/representation. This unified objective allows EZ-M to learn a task-conditioned world model and planner-aligned policy/value predictions end-to-end from replayed trajectories, while maintaining stable latent rollouts that are crucial for effective planning.

\subsection{Path Consistency} 
A critical limitation in the loss design of \citet{wang2024efficientzero} is that reward and value predictors are trained via independent supervision, lacking explicit constraints for mutual consistency along imagined rollouts. Consequently, the learned value $\hat{v}(h_k)$ often diverges from its recursive decomposition $\hat{r}(h_k, u_k) + \gamma \hat{v}(h_{k+1})$. This discrepancy accumulates over rollout steps, leading to compounding estimation errors during tree search. In the multi-task setting, this inconsistency is further exacerbated, as the shared model must simultaneously generalize across heterogeneous task dynamics and varying reward scales.

To address this issue, we incorporate a path-consistency regularizer to stabilize multi-task value learning under heterogeneous reward scales, inspired by~\citep{farquhar2021self}. For a task $\tau$, given a $K$-step trajectory segment $(s_k,a_k,r_k,s_{k+1})^{K-1}_{k=0}$ sampled from the corresponding replay buffer, we minimize the following objective
\begin{equation}
    \mathcal{L}^{pc}=\sum_{k=0}^{K-1}\mathcal{L}_\text{CE}\left(\text{sg}\left(\hat{r}(h_k,u_k)+\gamma\hat{v}(h_{k+1})\right), \hat{v}(h_k)\right)
\end{equation}
The stop-gradient operator $\text{sg}(\cdot)$ prevents destabilizing target coupling between reward and value learning, since the path consistency loss will never update the reward prediction network. It is a satisfying property, as the network can concentrate on the challenging reward prediction across tasks. The final training objective could be formulated as a summation over tasks:
\begin{equation}
    \mathcal{L}=\sum_{\tau\in\mathcal{T}}\mathcal{L}^\text{base}_\tau+\mathcal{L}^\text{pc}_\tau
\end{equation}

\subsection{Multi-Task Design Choices}
\textbf{Learnable task embedding}. To accommodate EZ-M in multi-task learning, we introduce a learnable task embedding to explicitly condition all major components of the agent on the task identity~\citep{hansen2023td}. Concretely, each task index is mapped to a trainable vector that is injected into the representation/dynamics/prediction modules via concatenation. This design provides a lightweight mechanism for task-specific specialization while keeping the majority of parameters shared, which is especially important when tasks exhibit distinct transition or similar patterns. 
In practice, the embedding enables the network to form task-aware latent states and value/reward predictions without requiring separate heads per task, and it also offers a clean interface for extending to new tasks by learning new embeddings while reusing the shared backbone.

\textbf{Action masking and observation padding}. Multi-task environments often differ in action space size, and observation dimensionality in proprio control, which can break a single unified network interface. We address this by using action masking and observation padding to standardize inputs/outputs across tasks. For actions, we maintain a global action space and apply a per-task binary mask to invalidate illegal actions during both policy evaluation and planning. For observations, we pad task-specific observations with all-zero tensors to a fixed maximum dimension. This simple unification avoids ad-hoc architecture branching, preserves vectorized training, and ensures that differences across tasks are handled explicitly and safely (invalid actions are never selected, and padded inputs do not leak spurious signals).

\textbf{Independent Experience Replay.} To mitigate task imbalance and minimize negative transfer arising from heterogeneous data distributions, we employ an independent experience replay strategy. In standard monolithic buffers, tasks with higher reward magnitudes or easier dynamics often dominate the prioritized sampling mechanism. In contrast, we maintain a dedicated replay buffer for each task, enabling explicit control over the inter-task sampling ratio. This decoupling ensures that complex or sparse-reward tasks are not marginalized by dominant ones. Furthermore, this design streamlines the architecture: sampling homogeneous batches facilitates consistent task conditioning and simplifies the management of task-specific action masks.

\section{Theoretical Analysis}
\label{sec:theory}

We formalize the structural advantage of Model-Based RL (MBRL) for online task scaling. We show that while Model-Free (MF) baselines suffer from low gradient similarities across relevant tasks in Fig.\ref{fig:grad_sim}, EZ-M leverages task diversity as a dynamics regularizer, leading to superior asymptotic sample efficiency. 

\subsection{Problem Formulation}
Consider a multi-task setting $\mathcal{T} = \{\tau_1, \dots, \tau_K\}$ sharing a physical environment. We define the global state space $\mathcal{S}$ and action space $\mathcal{A}$ as the union of reachable states and actions across all tasks (or a shared latent manifold). While individual tasks may utilize subsets of these spaces, the physical transition dynamics $P(s'|s,a)$ remain stationary. Each task $\tau_k$ is defined by a unique reward function $r_k(s,a)$. In EZ-M, we map these to a shared latent space via $h = \mathcal{H}(s)$ and $u = \mathcal{U}(a)$.

\subsection{Interference vs. Invariance}

\begin{proposition}[\textbf{Gradient Interference in MF} \citep{yu2020gradient}]
\label{prop:interference}
Let $\theta$ be the shared parameters of a model-free policy. If optimal policies for distinct tasks $\tau_i, \tau_j$ require disjoint actions in the same latent state, the expected cosine similarity of their gradients is negative:
\begin{equation}
    \mathbb{E} \left[ \langle \nabla_\theta \mathcal{L}_{\tau_i}, \nabla_\theta \mathcal{L}_{\tau_j} \rangle \right] < 0.
\end{equation}
\end{proposition}
\textit{Remark:} In standard MF-MTL, this necessitates gradient projection (e.g., PCGrad) or massive parameter scaling to disentangle conflicting task subspaces.

\begin{lemma}[\textbf{Dynamics Invariance}]
\label{lemma:invariance}
The physical transition dynamics $P$ are invariant across $\mathcal{T}$. Therefore, there exists a unique optimal parameter set $\phi^*$ for the dynamics model $\mathcal{G}_\phi$ that minimizes the prediction error for all tasks simultaneously:
\begin{equation}
    \phi^* \in \bigcap_{k=1}^K \arg\min_\phi \mathcal{L}_{\text{dyn}}^{(k)}.
\end{equation}
\end{lemma}
\textit{Remark:} Unlike policy learning, dynamics learning presents no conflicting objectives. While distinct tasks may require contradictory actions in the same state (e.g., Task A requires ``turning left'' vs. Task B requires ``turning right''), the physical consequence of any action remains objective and independent of task intent. Thus, task diversity effectively acts as a regularizer, forcing the representation to capture the causal structure of the environment rather than task-specific correlations.

\subsection{Sample Complexity Analysis}

Let $d_{\text{dyn}}$ and $d_{\text{rew}}$ denote the sample complexity to learn the transition dynamics and a single reward function.

\textbf{Assumption 1 (Dynamics Dominance).} In high-dimensional robotic control, learning the physical dynamics is significantly more sample-intensive than learning a task-specific reward: $d_{\text{dyn}} \gg d_{\text{rew}}$.

\textit{Remark:} This is consistent with theoretical findings in reward-free exploration~\citep{jin2020reward, agarwal2020flambe}, which establish that learning the transition operator governs the minimax sample complexity (scaling with $O(S^2 A)$), whereas reward estimation is comparatively cheap (scaling with $O(SA)$).

\begin{theorem}[\textbf{Asymptotic Task Scaling Efficiency}]
\label{thm:efficiency}
Let $N_{\text{MF}}$ and $N_{\text{MB}}$ denote the total sample complexity to master $K$ tasks. Under worst-case interference for MF and shared dynamics for MB, the per-task sample complexity satisfies:
\begin{equation}
    \lim_{K \to \infty} \frac{N_{\text{MB}}(K)}{K} = d_{\text{rew}} \ll \lim_{K \to \infty} \frac{N_{\text{MF}}(K)}{K} = d_{\text{dyn}} + d_{\text{rew}}.
\end{equation}
\end{theorem}

\begin{proof}[Proof]
Model-free agents implicitly entangle dynamics and reward learning. Due to gradient interference (Proposition \ref{prop:interference}), they fail to share structural knowledge effectively, solving each control problem as an independent instance: $N_{\text{MF}} \propto K(d_{\text{dyn}} + d_{\text{rew}})$.
In contrast, EZ-M decouples these processes. The invariant dynamics $\mathcal{G}$ are learned once (Lemma \ref{lemma:invariance}). For any additional task $k > 1$, the agent leverages the mature world model and only incurs the cost of learning the reward head. Thus, $N_{\text{MB}} \approx 1 \cdot d_{\text{dyn}} + K \cdot d_{\text{rew}}$. Dividing by $K$ and taking the limit yields the result.
\end{proof}
\section{Experiments}
In this section, we outline the experiments that evaluate our proposed EZ-M approach. The experiments are designed to answer the following questions: (1) How does the performance of EZ-M compare to mainstream baselines? (2) How does model performance change with the number of tasks trained simultaneously? (3) How do world model components share knowledge between tasks? (4) What is the necessity of proposed methods, such as path consistency, task embedding, and independent experience replay? The answers to these questions are listed in the following subsections, and more details and curves of our experiment can be found in Appendix.\ref{app:env-perf-curves}.

\subsection{Experimental settings}
\textbf{Benchmarks}. We evaluate EZ-M on HumanoidBench~\citep{sferrazza2024humanoidbench}, a challenging suite for whole-body humanoid control that requires coordinating dozens of degrees of freedom under contact-rich dynamics. Compared to standard locomotion benchmarks like DMControl~\citep{tassa2018deepmind} saturated with baselines, HumanoidBench involves difficult tasks like stair, hurdle, and maze that still remain a considerable gap to be solved. Following ~\citep{nauman2025bigger}, we consider two difficulty tiers, \textbf{Medium} and \textbf{Hard}. Medium uses the H1-* setting (no hand control) with 9 locomotion tasks: stand, walk, run, stair, crawl, slide, pole, hurdle, and maze. Hard uses H1hand-* (with hand control) with 14 tasks: the same 9 locomotion tasks plus sit-simple, sit-hard, balance-simple, balance-hard, and reach. Environment steps of both settings are limited to 1 million. The action-repeat is not applied.

\textbf{Baselines.} We compare EZ-M against established baselines spanning single-task (ST) and multi-task (MT) regimes, as well as model-free (MF) and model-based (MB) approaches. Key baselines include \textbf{TD-MPC2} (ST, MB)~\citep{hansen2023td}, \textbf{DreamerV3} (ST, MB)~\citep{hafner2023mastering}, \textbf{MH-SAC} (MT, MF)~\citep{yang2020multi}, and \textbf{BRC} (MT, MF)~\citep{nauman2025bigger}. Baseline scores are sourced directly from published results where available; otherwise, we report results obtained from re-evaluating official code. Notably, BRC relies on a massive 1B-parameter model to achieve its performance. In contrast, EZ-M utilizes a compact architecture with only 16M parameters, demonstrating that structured world models can achieve superior sample efficiency without brute-force parameter scaling.

\textbf{Evaluation Metrics.} To assess aggregate performance, we report normalized scores calculated as detailed in Appendix~\ref{app:score-norm}. Additionally, we provide full training curves plotting original returns against environment steps in Figure~\ref{fig:medium-unnorm-curves} and Figure~\ref{fig:hard-unnorm-curves} (Appendix~\ref{app:env-perf-curves}).

\subsection{Performance Analysis \& Knowledge Sharing}
EZ-M achieves state-of-the-art (SoTA) performance on the challenging HumanoidBench suite. As illustrated in Figure~\ref{fig:medium-unnorm-curves}, EZ-M demonstrates superior sample efficiency across the 9 medium-difficulty tasks compared to baselines. Furthermore, it maintains this dominance in the higher-dimensional regime, outperforming competitors on the 14 contact-rich H1Hand tasks (Figure~\ref{fig:hard-unnorm-curves}). Quantitatively, EZ-M effectively solves or achieves SoTA results in 7 out of 9 tasks in the Medium setting (Table~\ref{tab:humanoidbench_medium}) and 10 out of 14 tasks in the Hard setting (Figure~\ref{fig:hard-unnorm-curves}), highlighting its robustness in complex control domains.

\textbf{Gradient Similarity Analysis.} To investigate the underlying mechanisms, we analyze inter-task gradient similarity, comparing EZ-M's World Model components (Dynamics, Reward, Policy/Value) against BRC's Actor/Critic. We select (\textit{h1-walk}, \textit{h1-run}) as a \textit{related} pair and (\textit{h1-walk}, \textit{h1-crawl}) as a \textit{distantly related} one. As shown in Figure~\ref{fig:grad_sim}, EZ-M consistently maintains high similarity for the related pair (darker orange lines) while keeping it low for the distant one (lighter orange lines). It indicates that EZ-M effectively exploits positive transfer while mitigating interference. In contrast, BRC exhibits consistently lower similarity even for related tasks, suggesting that the model-based architecture facilitates superior knowledge transfer via shared dynamics.

\textbf{Failure Case Analysis.} While EZ-M dominates in locomotion and dynamic manipulation tasks, we observe a performance gap in stationary tasks like \textit{sit-hard} compared to BRC (Table \ref{tab:humanoidbench_hard}). We hypothesize that the intricate self-contact dynamics in sitting poses are harder for the world model to predict accurately, leaving room for future improvements in contact modeling.

\textbf{Computational Efficiency.} Despite the inherent overhead of MCTS planning, EZ-M maintains high efficiency through an asynchronous distributed architecture~\citep{ye2021mastering} that decouples rollout generation from model optimization. For the Medium task suite, EZ-M requires approximately \textbf{10 hours} on 2 NVIDIA A40 GPUs. In contrast, the strongest baseline, BRC, requires more than \textbf{40 hours} on a single A40 GPU to reach the same step count. These results demonstrate that EZ-M's parallelized design effectively offsets its planning costs, yielding a significantly faster research turnaround than sequential model-free baselines.

\begin{figure*}[h]
    \centering
    \vskip -.2cm
    \includegraphics[width=0.9\linewidth]{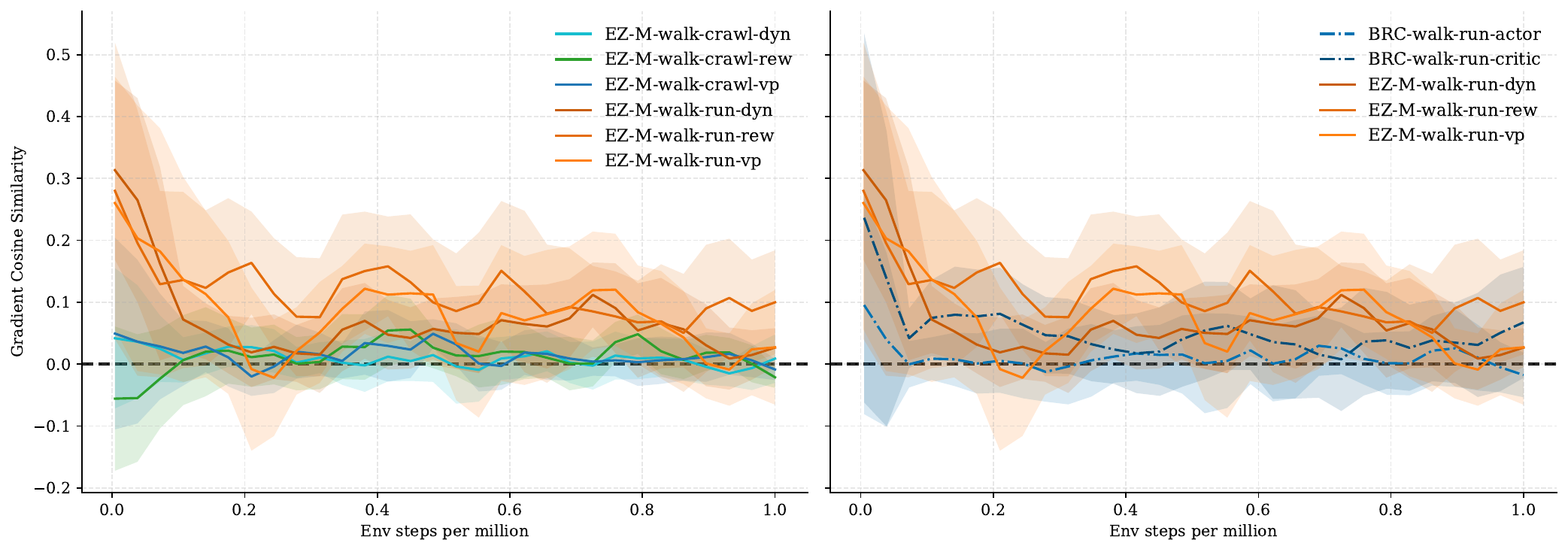}
    \caption{\textbf{Gradient similarities over training process}. We compare the gradient similarities on different model modules between BRC and EZ-M. We choose two task pairs, (\textit{h1-walk}, \textit{h1-run}) and (\textit{h1-walk}, \textit{h1-crawl}), representing \textit{relevant} and \textit{irrelevant} tasks, respectively. \textit{dyn}, \textit{rew}, and \textit{vp} represent the dynamics, reward, and value-policy model in EZ-M. \textbf{(Left)} We validate that the gradient similarities between the relevant task pair are higher than the irrelevant across the training process, indicating positive knowledge transfer. \textbf{(Right)} We showcase that EZ-M modules have higher gradient similarities than BRC modules in the relevant task pair. Curves are slightly smoothed for better visualization.}
    \label{fig:grad_sim}
    \vskip -.2cm
\end{figure*}

\subsection{Performance scaling with the number of tasks}
\begin{figure}[h]
    \vskip -.2cm
    \centering
    \includegraphics[width=0.8\linewidth]{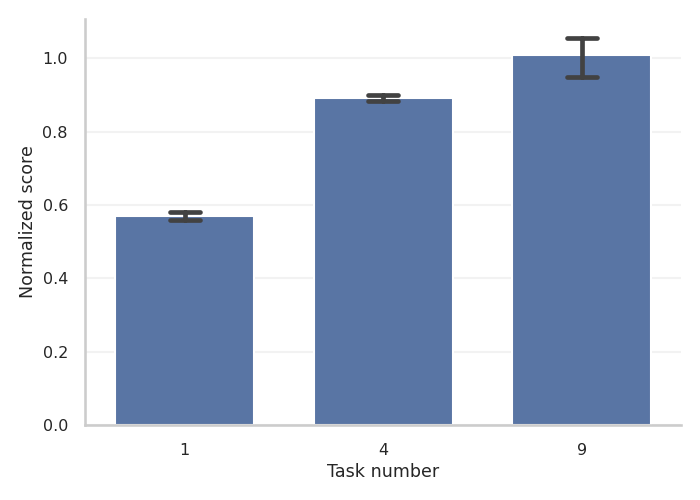}
    \caption{\textbf{Performance scale with the number of tasks.} \textit{y}-axis denotes the task-average, normalized episodic return and \textit{x}-axis represents number of tasks in training simultaneously. We use 3 different random seeds and report 95\% confidence intervals.}
    \label{fig:task_scaling}
    \vskip -.2cm
\end{figure}

We further investigate the scaling properties of EZ-M with respect to the number of concurrent tasks. We evaluate training configurations with task cardinalities of $N \in \{1, 4, 9\}$, where $N=1$ serves as the Single-Task baseline. For the $N=4$ setting, we partition the tasks into two disjoint subsets: \{\textit{h1-stand}, \textit{h1-walk}, \textit{h1-run}, \textit{h1-stair}\} and \{\textit{h1-crawl}, \textit{h1-pole}, \textit{h1-slide}, \textit{h1-hurdle}\}. To strictly isolate the benefit of multi-task scaling, we report the average performance on these 8 tasks across all settings. As illustrated in Figure~\ref{fig:task_scaling}, EZ-M exhibits a positive scaling law: performance consistently improves with the number of tasks given the same total environment step budget. This trend empirically validates the efficacy of multi-task online RL as a mechanism for amplifying sample efficiency.

\subsection{Ablation study}
\begin{figure}[h]
    
    \centering
    \includegraphics[width=1.0\linewidth]{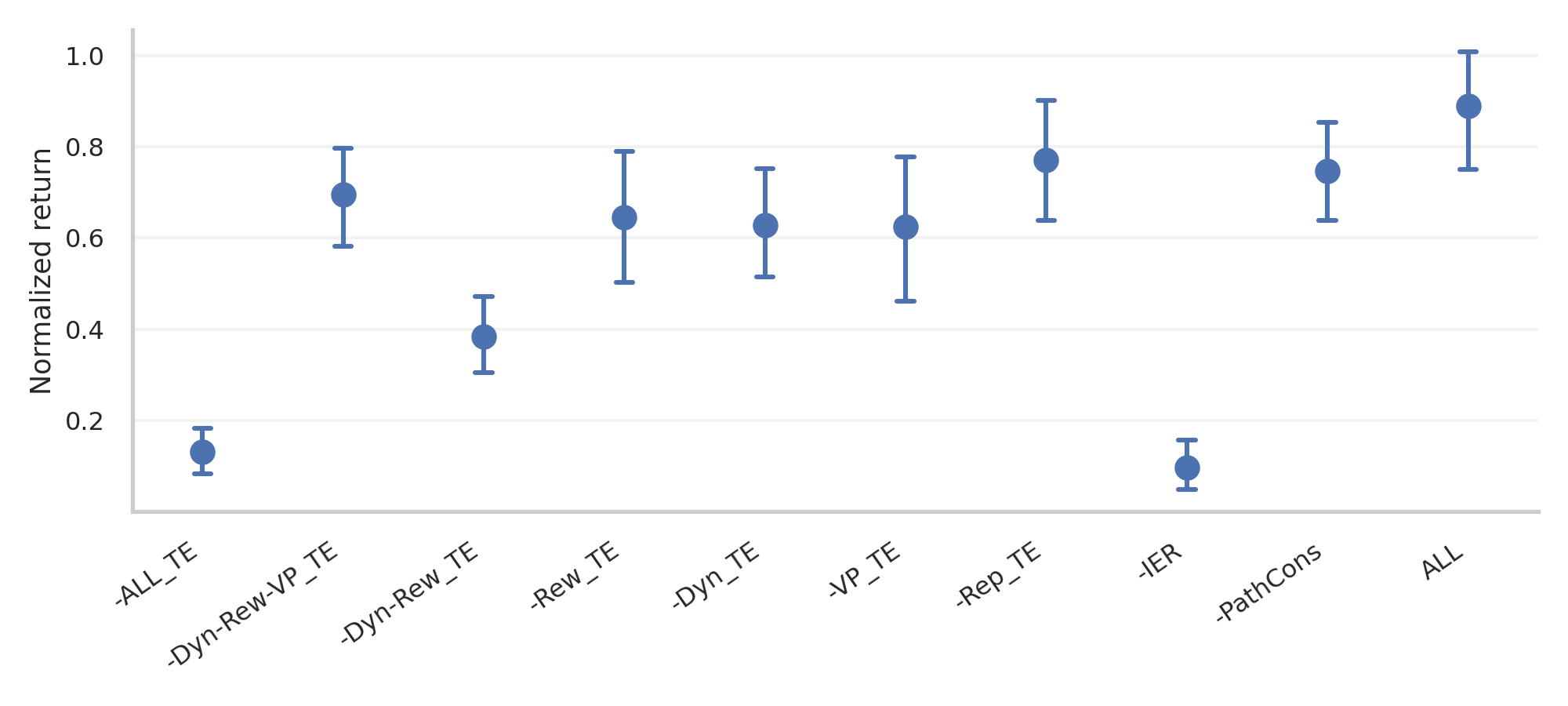}
    \caption{\textbf{Ablations on model components.} \textit{y}-axis denotes the task-average, normalized episodic return and \textit{x}-axis represents different settings. $-$ represents removing a component. \textit{TE} represents task embedding. \textit{Dyn}, \textit{Rew}, \textit{VP}, \textit{Rep} represent the dynamics, reward, value-policy, representation model, respectively. -\textit{Dyn}-\textit{Rew}-\textit{VP}\_\textit{TE} indicates removing the task embedding on dynamics, reward, and value-policy models. \textit{IER} represents independent experience replay. \textit{PathCons} is the path consistency loss. We use 3 different random seeds and report 95\% confidence intervals.}
    \label{fig:ablation}
    \vskip -.2cm
\end{figure}

Finally, we perform ablation studies to isolate the individual contributions of the core mechanisms introduced in EZ-M: Path Consistency (PC) loss, Independent Experience Replay (IER), and Task Embeddings. Given the architectural depth of world models, we additionally investigate the optimal injection points for task conditioning. All experiments are conducted on the HumanoidBench-Medium setting with a 1M step budget. As illustrated in Figure~\ref{fig:ablation}, \textbf{Independent Experience Replay} emerges as the most critical component; its removal leads to the most severe performance degradation, underscoring its role in mitigating data imbalance. The \textbf{Path Consistency loss} also provides substantial gains by enforcing alignment between reward and value predictions. Furthermore, our structural analysis confirms that task embeddings are indispensable, particularly when conditioning the Dynamics and Reward models to capture task-specific physics and objectives.

\section{Discussion}
In this work, we demonstrated the efficacy of Model-Based Reinforcement Learning for scaling online multi-task robotic control. Our results on HumanoidBench validate that task scaling serves as a powerful regularizer: by leveraging invariant physical dynamics, EZ-M converts diverse multi-task data into robust representations, achieving sample efficiency superior to model-free baselines. This confirms that shared world models effectively bridge distinct behaviors while mitigating the gradient interference that plagues policy-centric approaches. Limitations remain regarding the computational overhead of tree search, which may pose latency challenges for high-frequency real-world deployment. Additionally, our framework assumes a shared physical embodiment; transfer efficacy may diminish across tasks with disjoint dynamics or heterogeneous morphologies. Future work will focus on (1) distilling the planner into a lightweight policy for real-time inference, and (2) integrating language-conditioned goals to extend EZ-M toward open-ended instruction following.

\bibliography{example_paper}
\bibliographystyle{icml2026}

\newpage
\appendix
\onecolumn
\section{Training Curves}
\label{app:env-perf-curves}

Fig.~\ref{fig:medium-unnorm-curves} and Fig.~\ref{fig:hard-unnorm-curves} report the training dynamics on the Medium and Hard task suites, respectively.
We plot the evaluation return against environment steps, where each curve is aggregated over multiple random seeds (with shaded regions indicating the variability across runs).
These plots complement the main paper summary by revealing both learning speed and stability throughout training.

\begin{figure}[h]
    \centering
    \includegraphics[width=1.0\linewidth]{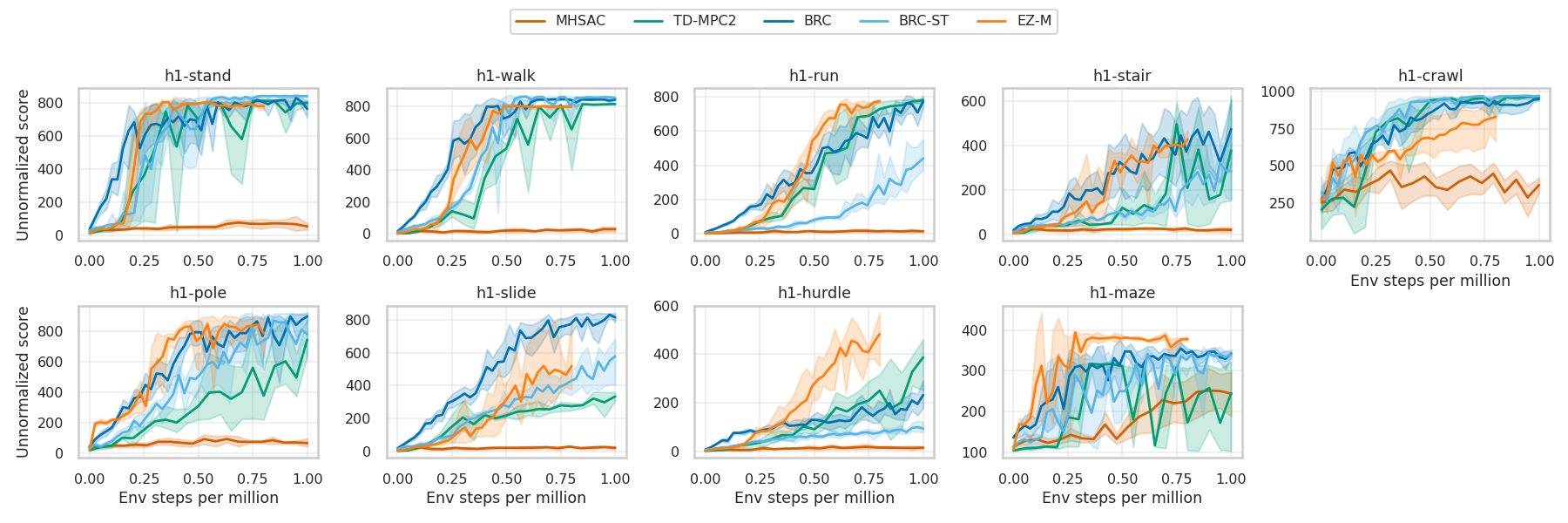}
    \caption{\textbf{Unnormalized curves on HumanoidBench-Medium benchmark.} \textit{y}-axis denotes the episodic return and \textit{x}-axis represents environment interactions per million. All experiments use 3 different random seeds. We report 95\% confidence intervals.} 
    \label{fig:medium-unnorm-curves}
\end{figure}

\begin{figure}[h]
    \centering
    \includegraphics[width=1.0\linewidth]{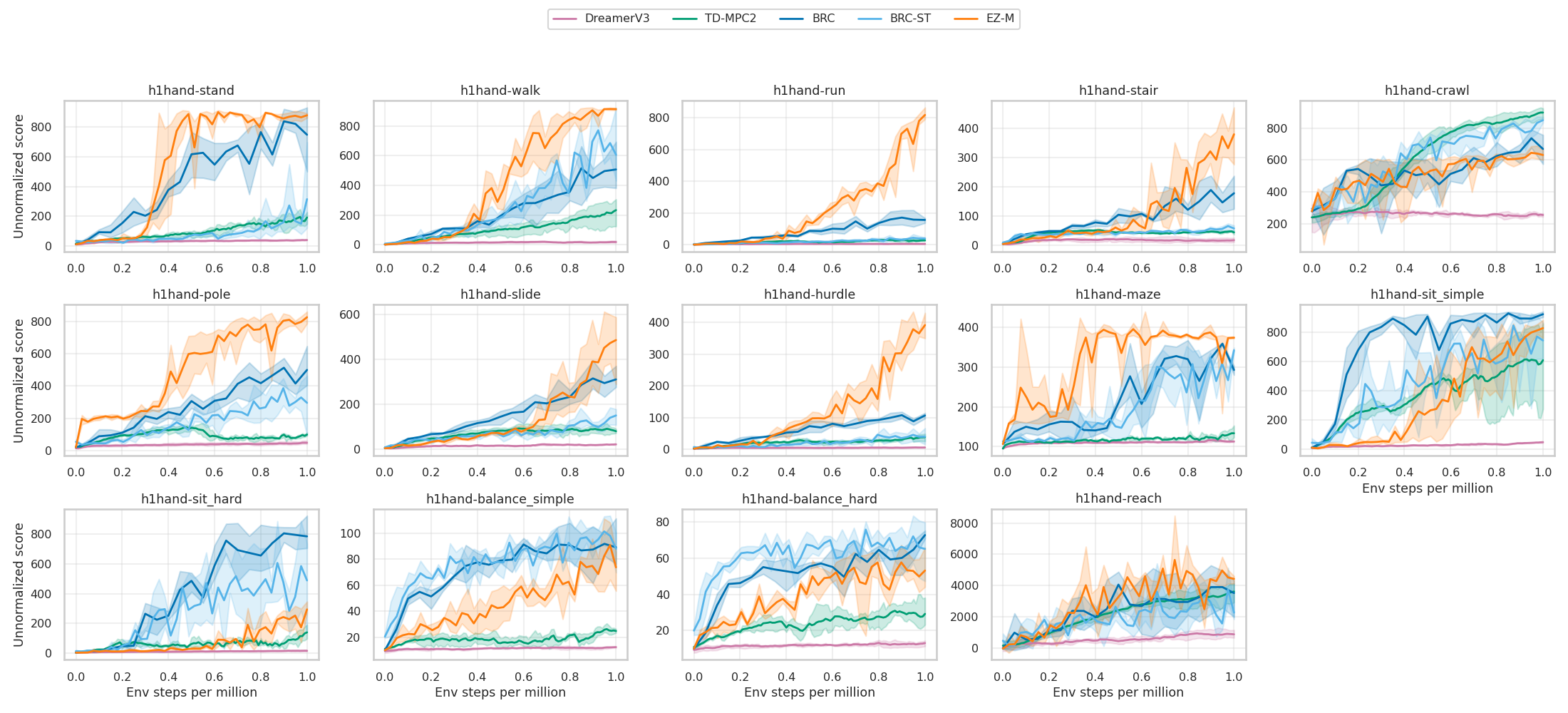}
    \caption{\textbf{Unnormalized curves on HumanoidBench-Hard benchmark.} \textit{y}-axis denotes the episodic return and \textit{x}-axis represents environment interactions per million. All experiments use 3 different random seeds. We report 95\% confidence intervals.}
    \label{fig:hard-unnorm-curves}
\end{figure}

\newpage
\section{Score Normalization \& Table Performance}
\label{app:score-norm}
We normalize episode returns using the following equation:
\begin{equation}
    \text{Return}_\text{norm}=\frac{\text{Return}-\text{Random Score}}{\text{Success Score}-\text{Random Score}}
\end{equation}
The random score and success score of each task could be found in Table.\ref{tab:humanoidbench_medium} and \ref{tab:humanoidbench_hard}.
\begin{table}[h]
\centering
\caption{\textbf{Scores for HumanoidBench-Medium (1M) tasks.} We \textbf{bold} the results surpass the success score or achieve the SoTA.}
\label{tab:humanoidbench_medium}
\begin{tabular}{lrrrrrrr}
\toprule
Task & Random & Success & MHSAC & TD-MPC2 & BRC & BRC-ST & EZ-M(ours) \\
\midrule
\texttt{h1-stand}  
& 10.545 & 800.0 
& 74.904  
& \textbf{810.688} 
& \textbf{829.358} 
& \textbf{841.246} 
& \textbf{806.034} \\

\texttt{h1-walk}   
& 2.377 & 700.0 
& 29.352  
& \textbf{814.249} 
& \textbf{844.784} 
& \textbf{858.926} 
& \textbf{802.969} \\

\texttt{h1-run}    
& 2.020 & 700.0 
& 17.434  
& \textbf{778.845} 
& \textbf{772.299} 
& 439.246 
& \textbf{775.867} \\

\texttt{h1-stair}  
& 3.112 & 700.0 
& 25.421  
& \textbf{497.040} 
& 474.852 
& 322.088 
& 459.234 \\

\texttt{h1-crawl}  
& 272.658 & 700.0 
& 449.088  
& \textbf{963.360} 
& \textbf{949.326} 
& \textbf{971.295} 
& \textbf{830.971} \\

\texttt{h1-slide}  
& 3.191 & 700.0 
& 29.004  
& 334.359 
& \textbf{831.837} 
& 578.478 
& 577.740 \\

\texttt{h1-pole}   
& 20.090 & 700.0 
& 93.175  
& \textbf{744.285} 
& \textbf{897.267} 
& \textbf{869.307} 
& \textbf{878.993} \\

\texttt{h1-hurdle} 
& 2.214 & 700.0 
& 19.038  
& 387.131 
& 232.031 
& 99.793 
& \textbf{542.286} \\

\texttt{h1-maze}   
& 106.441 & 1200.0 
& 251.297  
& 311.648 
& 354.939 
& 344.027 
& \textbf{380.795} \\
\bottomrule
\end{tabular}
\end{table}

\begin{table}[h]
\centering
\caption{Scores for HumanoidBench-Hard (1M) tasks. We \textbf{bold} the results surpass the success score or achieve the SoTA.}
\label{tab:humanoidbench_hard}
\begin{tabular}{lrrrrrrr}
\toprule
Task & Random & Success & DreamerV3 & TD-MPC2 & BRC & BRC-ST & EZ-M(ours) \\
\midrule
\texttt{h1hand-stand}  
& 11.973 & 800.0 
& 40.938  
& 196.511 
& \textbf{840.061} 
& 315.990 
& \textbf{891.780} \\

\texttt{h1hand-walk}   
& 2.505 & 700.0 
& 19.819  
& 234.061 
& 515.383 
& \textbf{771.471} 
& \textbf{919.960} \\

\texttt{h1hand-run}    
& 1.927 & 700.0 
& 6.580  
& 32.551  
& 171.790 
& 44.512  
& \textbf{818.397} \\

\texttt{h1hand-stair}  
& 3.161 & 700.0 
& 16.193  
& 46.537  
& 188.810 
& 64.864  
& \textbf{380.349} \\

\texttt{h1hand-crawl}  
& 278.868 & 800.0 
& 261.801 
& \textbf{897.086} 
& 735.992 
& \textbf{849.207} 
& 644.386 \\

\texttt{h1hand-slide}  
& 3.142 & 700.0 
& 20.625  
& 89.454  
& 313.723 
& 148.877 
& \textbf{535.404} \\

\texttt{h1hand-pole}   
& 19.721 & 700.0 
& 48.306  
& 98.703  
& 512.936 
& 386.245 
& \textbf{841.353} \\

\texttt{h1hand-hurdle} 
& 2.371 & 700.0 
& 6.381  
& 38.890  
& 108.067 
& 46.112  
& \textbf{405.686} \\

\texttt{h1hand-maze}   
& 106.233 & 1200.0 
& 116.332 
& 133.406 
& 358.803 
& 342.268 
& \textbf{380.642} \\

\texttt{h1hand-sit\_simple} 
& 10.768 & 750.0 
& 47.866  
& 612.656 
& \textbf{926.804} 
& \textbf{843.850} 
& \textbf{814.762} \\

\texttt{h1hand-sit\_hard}   
& 2.477 & 750.0 
& 15.589  
& 139.181 
& \textbf{805.502} 
& 605.966 
& 304.836 \\

\texttt{h1hand-balance\_simple} 
& 10.170 & 800.0 
& 12.615  
& 26.401  
& 91.749  
& \textbf{101.260} 
& 98.004 \\

\texttt{h1hand-balance\_hard}   
& 10.032 & 800.0 
& 13.110  
& 29.886  
& \textbf{72.918}  
& 72.199  
& 57.565 \\

\texttt{h1hand-reach}  
& -50.024 & 12000.0 
& 894.156 
& 3610.454 
& 3895.343 
& 3729.537 
& \textbf{5353.093} \\
\bottomrule
\end{tabular}
\end{table}

\newpage
\section{Hyperparameters \& Model architecture}\
Table~\ref{tab:key_hyperparams} summarizes the key hyperparameters used in our experiments.
Unless otherwise specified, we keep the same hyperparameter configuration across all tasks to ensure a fair comparison and to reduce per-environment tuning.

Table~\ref{tab:model_architecture} details the network architecture used by EZ-M on HumanoidBench-Hard.
The model follows a modular design with a shared representation backbone, a task-embedding conditioning, and separate dynamics/reward/value/policy prediction components.
We adopt lightweight MLP backbones with LayerNorm and residual blocks for stable optimization, while using discrete-support heads (51 bins) for reward and value prediction to improve learning signal quality.
Overall, the full network contains only \textbf{16M} parameters, which is orders of magnitude smaller than the \textbf{BRC} baseline with a \textbf{1B}-scale model, highlighting the parameter-efficiency of our approach.

\begin{table}[h]
\centering
\caption{Key hyperparameters used in HumanoidBench experiments}
\label{tab:key_hyperparams}
\begin{tabular}{ll}
\toprule
\textbf{Hyperparameter} & \textbf{Value} \\
\midrule
Environment & HumanoidBench \\
Discount factor $\gamma$ & 0.99 \\
Unroll steps & 5 \\
TD steps & 5 \\
Batch size & 512 \\
Replay buffer size & $1\times 10^{6}$ (Medium) / $2\times 10^{6}$ (Hard) \\
Training steps & $2\times 10^{5}$ (Medium) / $6\times 10^{5}$ (Hard) \\
\midrule
Network hidden dimension & 512 \\
Task embedding dimension & 128 \\
Task embedding usage & Rep / Dyn / Rew / Policy-Value \\
Number of res blocks & 3 \\
\midrule
MCTS simulations & 32 \\
Top actions & 16 \\
Sampled actions & 16 \\
\midrule
Optimizer & Adam \\
Learning rate & $1\times 10^{-4}$ \\
Weight decay & $2\times 10^{-5}$ \\
\bottomrule
\end{tabular}
\end{table}

\begin{table}[h]
\centering
\caption{EfficientZero network architecture (for HumanoidBench-Hard).}
\label{tab:model_architecture}
\begin{tabular}{@{}ll@{}}
\toprule
Module & Architecture \\ \midrule
Representation $f_\theta$ 
& MLP(305$\to$512) + LN + Tanh; 3$\times$ ResBlock(512) \\
Task embedding 
& Embedding(14, 128), max\_norm=1 \\
Dynamics $g_\theta$
& LN(701); MLP(701$\to$512$\to$512); 3$\times$ ResBlock(512) \\
Reward head $r_\theta$
& act-MLP(189$\to$64)+LN+Tanh; \\
& dyn trunk: LN(704)+MLP(704$\to$512$\to$512)+3$\times$ ResBlock(512); \\
& rew: [Linear(512$\to$512)+BN+ReLU]$\times$2; Linear(512$\to$51) \\
Value heads $v_\theta$
& trunk dim=640; 3$\times$ ResBlock(640); \\
& 5 heads: [Linear(640$\to$512)+Dropout+BN+ReLU]$\times$2; Linear(512$\to$51) \\
Policy head $\pi_\theta$
& [Linear(640$\to$512)+BN+ReLU]$\times$2; Linear(512$\to$122) \\
Projection
& Linear(512$\to$1024)+LN+ReLU; Linear(1024$\to$1024)+LN+ReLU; \\
& Linear(1024$\to$512)+LN \\
Proj. head
& Linear(512$\to$1024)+LN+ReLU; Linear(1024$\to$512) \\
\textbf{Total params} & \textbf{16M}\\
\bottomrule
\end{tabular}
\end{table}

\end{document}